\documentclass[]{article}

\usepackage{lineno,hyperref}
\modulolinenumbers[5]

\usepackage{amsmath,amsthm, amssymb}
\newtheorem{thm}{Theorem}
\newtheorem{mydef}{Definition}
\newtheorem{mylemma}{Lemma}

\newcommand{\cupdot}{\overset{\cdot}{\cup}}

\expandafter\def\expandafter\UrlBreaks\expandafter{\UrlBreaks%  save the current one
  \do\a\do\b\do\c\do\d\do\e\do\f\do\g\do\h\do\i\do\j%
  \do\k\do\l\do\m\do\n\do\o\do\p\do\q\do\r\do\s\do\t%
  \do\u\do\v\do\w\do\x\do\y\do\z\do\A\do\B\do\C\do\D%
  \do\E\do\F\do\G\do\H\do\I\do\J\do\K\do\L\do\M\do\N%
  \do\O\do\P\do\Q\do\R\do\S\do\T\do\U\do\V\do\W\do\X%
  \do\Y\do\Z}

\begin{document}

%\begin{frontmatter}

\title{Cost-Based Intuitionist Probabilities \\ on Spaces of Graphs, Hypergraphs and Theorems}

\author{Ben Goertzel \\
OpenCog Foundation, Hong Kong \\
ben@goertzel.org}
%\fntext[myfootnote]{Hanson Robotics}

%\author[mymainaddress,mysecondaryaddress]{OpenCog Foundation}
%\ead[url]{ben@goertzel.org}



\maketitle

\begin{abstract}
A novel partial order is defined on the space of digraphs or hypergraphs, based on assessing the cost of producing a graph via a sequence of elementary transformations.   Leveraging work by Knuth and Skilling on the foundations of inference, and the structure of Heyting algebras on graph space, this partial order is used to construct an intuitionistic probability measure that applies to either digraphs or hypergraphs.   As logical inference steps can be represented as transformations on hypergraphs representing logical statements, this also yields an intuitionistic probability measure on spaces of theorems.   The central result is also extended to yield intuitionistic probabilities based on more general weighted rule systems defined over bicartesian closed categories.
\end{abstract}

%\begin{keyword}
%ntuitionistic logic \sep hypergraphs \sep category theory
%\end{keyword}

\section{Introduction}

Standard probability theory is founded on set theory.   In real world applications, however, one often wants to assess probabilities of entities other than sets -- such as graphs, hypergraphs or logical theorems.   One can address such cases by reducing these other entities to sets, but it's not clear that this is always the best approach.   It is also  interesting to explore the possibility of founding an alternate version of probability theory on category theory, which can then be used to more elegantly and naturally model non-set entities.  

This leads to the question: What properties does a  category need to have to support a "probability theory" worth of the name?    A general answer  is not yet known; but it is clear that Heyting algebras (bicartesian closed categories that are also residuated lattices) are at least sufficient to the task.   It is known that Heyting structure can be used to construct "intuitionistic probability" \cite{georgescu2010probabilistic} that  constitutes an uncertain version of the intuitionist logic that is modeled by Heyting structure.

As well as corresponding to intuitionistic logic, Heyting algebra also models the category of digraphs (and the category of hypergraphs).  Graphs and hypergraphs, alongside probability, play a key role in complex systems modeling, artificial intelligence and many other disciplines.  Logic systems may be straightforwardly modeled in terms of hypergraphs.   From this standpoint, having an "intuitionistic probability" that corresponds naturally to the algebra of graphs and hypergraphs may potentially be quite valuable.

Here we show how to construct a probability-like valuation on categories of graphs (from here on, to simplify terminology, in informal discussion we will use the generic term "*graphs" to refer to undirected graphs, digraphs or hypergraphs, using a more specific term when needed).   Our route is quite different from that of \cite{georgescu2010probabilistic}, and involves constructing a natural Heyting algebra structure on *graphs and then arguing that one can place a valuation on these that, when combined with the Heyting operations, obeys the symmetries that Knuth and Skilling \cite{knuth2012foundations} have identified as critical for deriving a notion of probability.   

The chief novelty of our development is the introduction of a new cost measure on the space of *graphs, based (roughly speaking) on minimizing the number of elementary operations required to generate a *graph from an "indecomposable" *graph.    This cost measure is used to define a novel partial order on *graph space, which is shown to cooperate nicely with standard lattice operations on *graphs.   First we consider the case where the elementary operations involved are elementary homomorphisms; then we generalize to the case where the operations are drawn from an arbitrary system of *graph transformations.   The latter, general case enables us to encompass hypergraphs representing collections of logical statements, on which logical inferences are represented as hypergraph transformations.  In this context, our *graph probabilities yield a novel, intuitionistic probability distribution over the space of logical theorems.

Finally, we present a further generalization of our main result, which applies beyond *graphs to any bicartesian closed category endowed with a distinguished set of morphisms, to be considered as a system of (possibly weighted) transformations.

\section{Knuth and Skilling's Foundations of Inference}

Knuth and Skilling, in \cite{knuth2012foundations}, identify a number of basic "symmetries" as being important for the derivation of the concept of probability.   
Where $m$ denotes a valuation mapping from a lattice into the unit interval, they posit that

$$
m(x \sqcup y) = m(x) \oplus m(y)
$$
$$
m(x \times y) = m(x) \otimes m(y)
$$

\noindent -- and then pose the question of what the operators $\oplus$ and $\otimes$ are.   They show that, under a number of reasonable assumptions about the lattice itself, $\oplus$ must be a monotone transformation of +, and $\otimes$ must be a monotone transformation of *.

The first symmetries they introduce pertain to the interaction between the order $<$and the valuation $m$ and the join operator. 

$$
x < y \rightarrow m(x) < m(y)
$$
$$
x<y  \rightarrow x \sqcup z < y \sqcup z
$$
$$
x<y  \rightarrow z \sqcup x < z \sqcup y
$$

\noindent (let us call the latter two the "join/order symmetries").

Then there are some basic symmetries corresponding to the lattice operations:

$$
(x\sqcup y) \sqcup z = x \sqcup (y \sqcup z)
$$

$$
(x\times y) \sqcup (y \times t) = (x \sqcup y) \times t
$$

$$
(u \times v) \times w = u \times (v \times w)
$$

Finally, they introduce the notion of a "predicate-context interval" $[x,t]$, and a function $p$ assigning each such interval a number (which is to be interpreted basically as $p(x|t)$).   Where $\circ$ is a composition operator, they assume

$$
p( [x,y] \circ [y,z]) = p([x,y]) \odot p([y,z])
$$
$$
([x,y] \circ [y,z]) \circ [z,t] = [x,y] \circ ([y,z] \circ [z,t])
$$

These rather basic axioms turn out to be enough to demonstrate that  $\oplus$ and $\otimes$  are monotone transformations of + and * ; and that $\odot$ is also *.   This is striking and much more elegant than Cox's original axiomatic derivation of probability theory, though K\&S follow Cox in that, underneath all the abstraction, at the heart of their proof lie calculations regarding real-valued functional equations.

\section{A Novel, Cost-Based Heyting Algebra Structure for Digraphs and Hypergraphs}

It is well known that directed graphs can be given a Heyting algebra structure \footnote{Note that, for technical reasons, in this context one generally works with digraphs that can contain self-loops (edges pointing from a vertex to itself)}.   A Heyting algebra is a bicartesian closed category that is also a poset, and whose sums are all finite; or, equivalently, it is is a bounded lattice equipped with a binary operation $a \rightarrow b$ of implication such that $c \sqcap a \leq b$ is equivalent to $c \leq (a \rightarrow b)$.  

To create a Heyting algebra from digraphs, typically the join operation $\sqcup$ of the lattice is taken as the disjoint union, and the meet operation $\sqcap$ is taken as the direct product.   This is basically an application of the structure of closed cartesian categories to digraphs.  What is subtler is the definition of the partial-order operation.   

In the standard approach \cite{hahn1997graph} \cite{gray2014digraph},  for finite digraphs $G$ and $H$, it is said that $G \leq H$ or $G \rightarrow H$ if there is a homomorphism from $G$ to $H$.   However, this definition of $\leq$ is problematic in that, while it is reflexive and transitive, it is not antisymmetric.  There exist digraphs $G$ and $H$ where $G \leq H$ and $H \leq G$ such that $G \neq H$.   So $\leq$ is only a quasiorder, not a true partial order.

Gray \cite{gray2014digraph} works around this by defining $G$ and $H$ to belong to the same equivalence class if both $G \leq H$ and $H \leq G$ hold.   Between equivalence classes $\leq$ is then a partial order.   But this is not always desirable in practice -- in many applications one cares about the difference between, say, the two cyclic graphs $C_2$ and $C_4$, even though the two are equivalent in Gray's sense.

An alternative is to deal with cases where $G \leq H$ and $H \leq G$ but $G \neq H$, by asking which of the two graphs $G$ and $H$ is actually the "simpler" one.   Lemma 2.14 in \cite{gray2014digraph} leads us toward a natural way of doing this.   It states that any graph homomorphism can be represented as a composition of elementary homomorphisms, where each elementary homomorphism identifies a single pair of nonadjacent vertices.   We can then define

\begin{mydef}
The {\bf cost} $c(G,H)$ to map $G \rightarrow H$ is defined as the minimum number of elementary homomorphisms that must be composed in order to build a homomorphism mapping between $G$ and $H$.  
\end{mydef}

Then, suppose we let the "source set" of $G$ mean the set of $F$ that have the minimum number of nodes, subject to the restriction that $F$ and $G$ are homomorphic.  Based on this we can define a partial order on the space of digraphs via:

\begin{mydef}
In the {\bf cost-based order} on digraphs, we say that $G < H$  if 
\begin{itemize}
\item there is a homomorphism from $G$ to $H$ 
\item for all $F$ in the source set of $G$ (which is also the source set of $H$), $c(F,G)< c(F,H)$
\end{itemize}
\end{mydef}

\noindent It is evident that, by this definition, we can't have $G < H$ and $H < G$ both.   Inspecting the relevant proofs in \cite{gray2014digraph}, it is clear that what we get from this is a Heyting algebra structure on the space of digraphs itself, rather than on a space of equivalence classes of digraphs.

It is also clear that the cost-based order fulfills the Knuth and Skilling join-order symmetries; i.e.

\begin{mylemma} \label{thm:costlemma}
Where $<$ denotes cost-based order and $\sqcup$ denotes disjoint-union,

$$
G<H \rightarrow G \sqcup A < H \sqcup A
$$    
\end{mylemma}

\begin{proof}
Clearly any homomorphism from $G$ to $H$ can be extended into a homomorphism from  $G \sqcup A$ into $H \sqcup A$, simply by extending it to keep $A$ constant.

Due to the nature of the disjoint union, the common source set of $G \sqcup A$ and $H \sqcup A$, is going to consist of the disjoint union of the common source set of $G$ and $H$, with the source set of $A$.   Letting $ss(X)$ denote the source set of $X$, we have for instance $ss(G \sqcup A) = ss(G \cupdot A) = ss(G) \cupdot ss(A)$.    

But then, if the shortest sequence of elementary homomorphisms from $ss(G)$ to $G$ is shorter than the shortest sequence of elementary homomorphisms from $ss(G)$ to $H$, it is clear that the shortest sequence of elementary homomorphisms from  $ss(G) \cupdot ss(A)$ to $G \cupdot A$ will be shorter than the shortest sequence of elementary homomorphisms from $ss(H) \cupdot ss(A)$ to $H \cupdot A$.   Because any sequence of elementary homomorphisms from $ss(G) \cupdot ss(A)$ to $G \cupdot A$ can be decomposed into a sequence of elementary homomorphisms from $ss(G)$ to $G$, followed by a sequence of elementary homomorphisms from $ss(A)$ to $A$; and similarly if one substitutes $H$ for $G$.
\end{proof}

Finally, while the results of Gray \cite{gray2014digraph} on which Lemma \ref{thm:prob} draws deals directly only with digraphs, it is straightforward to construct a comparable category of hypergraphs, and a comparable notion of hypergraph homomorphisms (see \cite{dorfler1980category}).   One may then verify that the above theorem holds for hypergraphs as well.

\section{A Cost-Based Probability Measure on Digraphs and Hypergraphs}

Putting Lemma \ref{thm:costlemma} together with the basic symmetries that come with the lattice structure of a Heyting algebra, we then find a recipe for constructing probability measures on digraphs and hypergraphs:

\begin{mylemma} \label{thm:prob}
Suppose one constructs a Heyting algebra on digraphs (allowing loops that self-connect vertices), using disjoint-union for $\sqcup$, direct-product for $\sqcap$, and cost-based order for the partial-order $<$.   Suppose one has a valuation $m'$ mapping digraphs into real numbers for which

$$
G < H \rightarrow m'(G) < m'(H)
$$

\noindent Then there is some monotone rescaling $m(x) = f(m'(x))$ so that

$$
p(G | H) = m(G \sqcap H) / m(H)
$$
$$
m(G \sqcup H) = m(G) + m(H) - m(G \sqcap H)
$$

\noindent ; and so that for $G$ and $H$ that are "independent" in the sense of being incomparable via the cost-based partial order, we also have:

$$
m(G \sqcap H) = m(G) * m(H)
$$

Similar statements hold for hypergraphs with self-connected vertices.
\end{mylemma}

What remains is to identify what the valuation $m'$ actually is.   One possibility is 

\begin{mydef}
The {\bf cost-based valuation}  of $G$, c(G), is defined as the minimum of $c(F,G)$ where $F$ is any  member of the source set of $G$.
\end{mydef}

We then have

\begin{thm} \label{thm:costprob}
{\bf Cost-Based Intuitionist Probability: Digraph/Hypergraph Homomorphism Version}.

If $<$is the cost-based order and $m'$ is the cost-based valuation, then (where G and H are digraphs or hypergraphs) we have

$$
G < H \rightarrow m'(G) < m'(H)
$$

\noindent and thus we can use $m'$ to construct a probability on digraphs or hypergraphs according to Lemma \ref{thm:prob}.
\end{thm}

\noindent where the proof follows immediately from the definitions of cost-based order and cost-based valuation.

Note that the probability measure constructed in $m$ is in essence a non-normalized probability; that is, none of the algebraic formulas derived regarding $m$ make any assumption on the scale of $m$.   However, the setting of \cite{knuth2012foundations} , on which Theorem \ref{thm:costprob} rests, is a finite universe, so that $m$ can always be linearily normalized into $[0,1]$.   In a *graph context, this equates to positing some finite, potentially very large digraph or hypergraph and then assuming we are looking at its subgraphs.

The probability constructed here appears to the author not to be identical to the intuitionistic probability constructed in \cite{georgescu2010probabilistic}, although the relationship has yet to be analyzed rigorously.   Generally, the properties of the "intuitionistic probability" thus constructed on digraphs and hypergraphs remain to be studied.   However, from a practical and computational perspective, the procedure to be followed to calculate the cost-based probability of a digraph or hypergraph is straightforward from the above.   

Various extensions to the above framework are also possible.   For instance one may consider the case where different elementary homomorphisms have different costs.  If each node is labeled with a cost value, then one can assign an elementary homomorphism a cost equal to the sum of the costs of the nodes involved; and one can assign a composite homomorphism a cost equal to the cost of the minimum-total-cost chain of elementary homomorphisms of which it can be composed.  This cost can be used to construct a cost-based order and cost-based valuation; and one can then proceed to construct a probability as in Theorem \ref{thm:costprob}.

\section{Generalizations: Rule Systems and Probabilities on Theorem Space}

We now consider a series of extensions of the above constructions, involving different cost-based partial orders, and leading to a broader array of applications.

First, a potentially important extension of the above ideas is to the case where, instead of elementary homomorphisms, one has some other set of *graph transformation rules in mind.      Relatedly, we explicitly wish to encompass the case of labeled *graphs, where nodes and links may be tagged with data such as types or numerical values.

Proceeding analogously to the definition of inference rules in the theory of Conceptual Graphs \cite{baget2002extensions} (but with some differences as well), we may define a useful variety of graph transformation "rules" as follows.  First, consider labeled *graphs, and assume a notion of inheritance on the graph labels (here we subsume floating-point weights as well as string labels like types, under the general notion of "label").   Then we may say $G <_\pi H$ if $G$ has the same structure as $H$, and the labels on $G$'s nodes and links inherit from the labels on the corresponding nodes and links of $H$.   In this case, one may construct a mapping $\pi$ from nodes/links of $H$ to the correspondent nodes/links who $G$.   Then, we may posit

\begin{mydef}\label{def:rule}
A {\bf rule} is a pair $(R,f)$, where $R$ is two-colored *graph $R$, with $R_0$ being the portion with color 0 and $R_1$ being the portion with color 1.   

Matching of $R$ to the "input" *graph $G$ occurs if $R_0 <_\pi G$.   $f$ is a function that assigns labels to the nodes and links of $R_1$ based on the labels on the nodes and links of $G$.   

The output of the rule is then obtained by the following process: 
\begin{itemize}
\item Create a new *graph $R_2$ with the structure of $R_1$, and the labels assigned by $f$ to the nodes and links of $R_1$
\item For every link $L$ between a node $n_0$ in $R_0$ and a node $n_1$ in $R_1$, create a corresponding link $L'$ between the node $\pi(n_0)$ in $G$ and the node in $R_2$ corresponding to $n_1$
\end{itemize}
\end{mydef}

Definition \ref{def:rule} differs a bit from the definition of rules in Conceptual Graphs, because in the latter case, the labels on the output graph $R_1$ don't depend on the specifics of the matched input graph $G$, so there is no need for $f$ and $R_2$.   The additional complexity given here is needed in order to account for labels that include uncertain truth value objects, and rules that modify the uncertain truth values of their conclusions based on the specific uncertain truth values of their premises.

Suppose one has a set $\mathcal{R}$ of rules of the sort described in Definition \ref{def:rule}, each with a certain numerical cost associated to it.  Then,   

\begin{mydef}
The {\bf cost} $c_\mathcal{R}(G,H)$ to map $G \rightarrow H$ using $\mathcal{R}$ is defined as the minimum-cost sequence of rules from $\mathcal{R}$ that can be used to build a mapping between $G$ and $H$.  
\end{mydef}

\noindent If one replaces the notion of "G is homomorphic to H" with the notion of "G can be transformed into H using a sequence of rules in $\mathcal{R}$", one can then define cost-based order and cost-based valuation analogously to what was done above.    A moment's reflection shows that Lemma \ref{thm:costlemma} still holds in this case, so that one can follow through all the way to Theorem \ref{thm:costprob} and construct a probability measure $m_\mathcal{R}$ within Theorem \ref{thm:prob}, defined relative to the rule-system $\mathcal{R}$; i.e.

\begin{thm} \label{thm:costprob}
{\bf Cost-Based Intuitionist Probability: Rule System Version}.

If $<_\mathcal{R}$is the cost-based order and $m'_\mathcal{R}$ is the cost-based valuation, both defined relative to rule system $\mathcal{R}$, then (where G and H are digraphs or hypergraphs) we have

$$
G <_\mathcal{R} H \rightarrow m'_\mathcal{R}(G) < m'_\mathcal{R}(H)
$$

\noindent and thus we can use $m'_\mathcal{R}$ to construct a probability $m_\mathcal{R}$ on digraphs or hypergraphs according to the argument of Lemma \ref{thm:prob}.
\end{thm}

In this latter construction, the probability of a *graph is proportional to the cost of the minimum-cost sequence of steps for building that graph from some minimal graph, using the given set of rules (where a minimal graph is one that can't be built from other graphs using the given set of rules).   The theorem states that this definition of probability plays nicely with the cost-based order (where cost is defined in terms of the given set of rules), the disjoint union and the direct product -- yielding a probability valuation that obeys the inclusion-exclusion formula, the product rule for the conjunction of independent terms, and the frequency-like characterization of conditional probability.

This extended version of our construction has apparent applications in the domain of logic.   Suppose one has hypergraphs representing logical statements, as in the OpenCog implementation (\cite{EGI1}, \cite{EGI2}) of Probabilistic Logic Networks \cite{PLN}.  In this case, one's logic rules are represented as hypergraph transformation rules of the sort specified in Definition \ref{def:rule}.   The construction given here then lets one define the probability of a logical statement, represented as a hypergraph, in terms of the minimal cost of constructing it using a given set of logic rules, beginning from "axioms" consisting of hypergraphs that can't be derived from any available data using the logic rules.     

This provides a novel angle on problem of how to place a meaningful and useful probability distribution on the space of theorems in some logic (for discussion of related problems, see \cite{demski2012logical} \cite{soares2014questions}).   By considering logic statements as collections of nodes links in a weighted, labeled hypergraph, and applying the current results regarding probability measures on hypergraphs, one gets a probability distribution on theorems.   In essence, in this distribution, the probability of a theorem is proportional to the length of the shortest proof of that theorem, which is simple and makes intuitive sense.   However, what we get from the current construction is a non-normalized, intuitionistic probability, rather than a standard probability.  Normalization based on assumption of a large finite containing hypergraph is feasible, and is fairly similar to restricting attention to proofs with specifically bounded length.   But generally there is much to be explored here.

One can also generalize even further.   Suppose one has a bicartesian closed category $\mathcal{C}$, with $\sqcup$ defined as disjoint union and $\sqcap$ defined as the categorial product ("direct product").   One can consider a mapping from objects of $\mathcal{C}$ to labels that lie in some partially-ordered space $\mathcal{L}$ (or one can model labels using monads, for example); then one can construct a definition of $<_\pi$ on subcategories of $\mathcal{C}$.   A rule is then a pair $( R, f)$, where $R$ is a two-colored subcategory of $\mathcal{C}$, and one can proceed as above.   

\bibliographystyle{alpha}
\bibliography{mybibfile}

\end{document}